\documentclass[11pt]{article}
\usepackage{amsfonts,amsthm,mathrsfs,amssymb}
\usepackage{amsmath}
\usepackage[page,title,titletoc,header]{appendix}
\usepackage{color}
\usepackage{authblk}

\usepackage{lipsum}

\newcommand\blfootnote[1]{%
	\begingroup
	\renewcommand\thefootnote{}\footnote{#1}%
	\addtocounter{footnote}{-1}%
	\endgroup
}

\usepackage[a4paper,hmargin=1.25in,vmargin=1in]{geometry}

\numberwithin{equation}{section}
\newtheorem{thm}{Theorem}[section]
\newtheorem{lemma}[thm]{Lemma}
\newtheorem{pro}[thm]{Proposition}
\newtheorem{corollary}[thm]{Corollary}

\newtheorem{Exa}{Example}[section]
\newtheorem{as}{Assumption}

\newcommand{\be}{\begin{equation}}
\newcommand{\ee}{\end{equation}}
\newcommand{\bea}{\begin{eqnarray*}}
\newcommand{\eea}{\end{eqnarray*}}
\newcommand{\mR}{\mathbb{R}}
\newcommand{\mN}{\mathbb{N}}
\newcommand{\mE}{\mathbb{E}}
\newcommand{\mcHK}{\mathcal{H}_K}

\newcommand{\mcL}{\mathcal{L}}
\newcommand{\mcE}{\mathcal{E}}
\newcommand{\mcF}{\mathcal{F}}

\newcommand{\tr}{\operatorname{tr}}
\newcommand{\la}{\langle}
\newcommand{\ra}{\rangle}

\newcommand{\eref}[1] {(\ref{#1})}

\usepackage[titletoc,title]{appendix}

\begin{document}
\title{Generalization Properties of Doubly Stochastic Learning Algorithms
}

\renewcommand\Affilfont{\tiny}
\author{Junhong Lin$^{\dag\qquad}$ Lorenzo Rosasco$^{\dag\ddag}$\\
{\scriptsize $^\dag$LCSL, Massachusetts Institute of Technology and Istituto Italiano di Tecnologia, Cambridge, MA 02139, USA}\\
{\scriptsize $^\ddag$DIBRIS, Universit\`a degli Studi di Genova, Via Dodecaneso 35, Genova, Italy}\\
}

\maketitle \baselineskip 16pt

\begin{abstract}
Doubly stochastic learning algorithms are  scalable kernel methods that perform very well in practice. However, their generalization properties are not well understood and their analysis is challenging since  the corresponding learning sequence may not be in the hypothesis space induced by the kernel.
In this paper, we  provide an in-depth theoretical analysis for different variants of doubly stochastic learning algorithms within the setting of nonparametric regression in  a reproducing kernel Hilbert space and considering  the square loss. Particularly, we derive convergence results on the generalization error for the studied algorithms either with or without an explicit penalty term.  To the best of our knowledge, the derived results for the unregularized variants are the first of this kind, while the results for the regularized variants improve those in the literature.
The novelties in our proof are  a  sample error  bound that requires controlling the trace norm of a cumulative operator, and a refined analysis of bounding initial error. 
\blfootnote{J. Lin is now with the \'{E}cole Polytechnique F\'{e}d\'{e}rale de Lausanne, Switzerland.}
\blfootnote{Emails: jhlin5@hotmail.com (J. Lin);  lrosasco@mit.edu (L. Rosasco).}

 \end{abstract}

\section{Introduction}
In nonparametric regression,  we are given a set of samples of the form $\{(x_i, y_i)\}_{i=1}^T$, where each $x_i \in \mR^d$ is an input, $y_i$ is a real-valued output, and the samples are drawn i.i.d. from an unknown distribution on $\mR^d \times \mR.$ The goal is to learn a function which can be used to predict future outputs based on the inputs.

Kernel methods \cite{shawe2004kernel,cucker2007learning,steinwart2008support} are a popular nonparametric technique based on choosing a hypothesis space to be  a reproducing kernel Hilbert space (RKHS).
Stochastic/online learning algorithms \cite{kivinen2004online,cesa-bianchi2004} (often called 
stochastic gradient methods \cite{robbins1951stochastic,polyak1992acceleration} in convex optimization) are among  the most efficient and fast learning algorithms.  At each iteration, they compute
a gradient estimate with respect to a new sample point and then updates the current solution by subtracting the scaled gradient estimate. In general, the computational complexities for training are $O(T + Td)$ in space and $O(T^2d)$ in time, due to the nonlinearity of kernel methods.
In recent years, different types of online/stochastic learning algorithms, either with or without an explicit penalty term, have been proposed and  analyzed, see e.g. \cite{cesa-bianchi2004,yao2006dynamic,ying2008online,shamir2013stochastic,tarres2014online, rosasco2015learning, dieuleveut2016harder,lin2017online} and references therein.

In  classic stochastic learning algorithms, all  sampling points  need being stored for testing. Thus, the implementation of the algorithm may be difficult in learning problems with high-dimensional inputs and large datasets. To tackle such a challenge, an alternative stochastic method, called  doubly stochastic learning algorithm was proposed in \cite{dai2014scalable}. The new algorithm is based on the random feature approach proposed in  \cite{rahimi2007random}. The latter result is based on Bochner's theorem and shows that most shift-invariant kernel functions can be expressed as an inner product  of some suitable random features. Thus the kernel function at each iteration in the original stochastic learning algorithm can be estimated (or replaced) by a random feature. As a result, the new algorithm allows us to avoid keeping all the sample points since it only requires generating the random features and recovers past random resampling them using specific random seeds \cite{dai2014scalable}. The computational complexities of the algorithm are $O(T)$ (independent of the dimension of the data) in space and $O(T^2 d)$ in time. Numerical experiments  given in \cite{dai2014scalable}, show that the algorithm is fast and comparable with state-of-the-art algorithms. Convergence results with respect to the solution of regularized expected risk minimization  were derived in \cite{dai2014scalable} for doubly stochastic learning algorithms with regularization, considering  general Lipschitz and smooth losses. 

In this paper, we study generalization  properties of doubly stochastic learning algorithms in the framework of nonparametric regression with the square loss.
Our contributions are  theoretical. First, for the first time, we prove generalization error bounds for doubly stochastic learning algorithms without regularization, either using a fixed constant step-size or a decaying step-size. Compared with the regularized version studied in \cite{dai2014scalable}, doubly stochastic learning algorithms without regularization do not involve the model selection of regularization parameters, and thus it may have some  computational advantages in practice. Secondly, we also prove generalization error bounds for doubly stochastic learning algorithms with regularization. Compared with the results in \cite{dai2014scalable}, our convergence rates are faster and  do not require the bounded assumptions on the gradient estimates as  in \cite{dai2014scalable},  see the discussion section for details.  The key ingredients to our proof are an error decomposition and an induction argument, which enables us to derive total error bounds provided that the initial (or approximation) and sample errors can be bounded. The initial and sample errors are bounded using  properties from integral operators and  functional analysis. The difficulty in the analysis is the estimation of the sample error, since the sequence generated by the algorithm may not be in the hypothesis space.
 The novelty in our proof is the estimation of the sample error involving upper bounding a trace norm of an operator, and a refined analysis of bounding the initial error.

 The rest of the paper is organized as follows. In the next section, we introduce the learning setting we consider and the doubly stochastic learning algorithms.
 In Section \ref{sec:generalization}, we present the main results on generalization properties for the studied algorithms and give some simple discussions.
Sections \ref{sec:error} to \ref{sec:deriving} are devoted to the proofs of all the main results.

\section{Learning Setting and Doubly Stochastic Learning Algorithms}\label{sec:doubly}
Learning a function from a given finite number of instances through efficient and practical algorithms is the basic goal of learning theory. Let the input space $X$ be a closed subset of Euclidean space $\mR^d$, the output space $Y \subseteq \mR$, and $Z=X\times Y.$ Let $\rho$ be a fixed Borel probability measure on $Z$, with its induced marginal measure on $X$ and conditional measure on $Y$  given $x \in X$ denoted by $\rho_X(\cdot)$ and $\rho(\cdot | x)$ respectively. In statistical learning theory, the Borel probability measure $\rho$ is unknown, but only a set of sample points $\mathbf z=\{z_i=(x_i, y_i)\}_{i=1}^T$ of size $T\in\mN$ is given. Here, we assume that the sample points are independently and identically drawn from the distribution  $\rho$.

 The quality of a function $f: X \to Y$ can be measured in terms of the expected risk with the square loss defined as
 \be\label{generalization_error}
\mcE(f) = \int_{Z} (f(x) - y)^2 d\rho(z).
\ee
In this case, the function minimizing the expected risk over all measurable functions is the regression function given by
\be\label{regressionfunc}
f_{\rho}(x) = \int_Y y d \rho(y | x),\qquad x \in X.
\ee
For any $f \in \mcL_{\rho}^2,$ it is easy to prove that
\be\label{excesserror}
\mcE(f) - \mcE(f_{\rho}) = \|f- f_{\rho}\|^2_{\rho}.
\ee
Here, $\mcL_{\rho}^2$ is the Hilbert space of square integral functions with respect to $\rho_X$, with its induced norm given by
$\|f\|_{\rho} = \|f\|_{\mcL_{\rho}^2} = \left (\int_X |f(x)|^2 d \rho_X(x)\right)^{1/2}$.
Throughout this paper we assume that $\int_{Y} y^2 d\rho< \infty.$ Thus, using \eref{excesserror} with $f=0$, $\mcE(f_{\rho}) + \|f_{\rho}\|_{\rho}^2$ is finite.

Kernel methods is based on choosing a hypothesis space as a reproducing kernel Hilbert space (RKHS). Recall that a reproducing kernel $K$ is a symmetric function $K: X
\times X \to \mR$ such that $(K(u_i, u_j))_{i, j=1}^\ell$ is
positive semidefinite for any finite set of points
$\{u_i\}_{i=1}^\ell$ in $X$. The kernel $K$ defines a RKHS $(\mathcal{H}_K, \|\cdot\|_K)$ as the
completion of the linear span of the set $\{K_x(\cdot):=K(x,\cdot):
x\in X\}$ with respect to the inner product $\la K_x,
K_u\ra_{K}:=K(x,u).$ For simplicity,
we assume that $K$ is a Mercer kernel, that is, $X$
is a compact set and $K : X\times X \to \mR$ is continuous. 

Online/stochastic learning is an important class of efficient algorithms to perform learning tasks. Over the past few decades, several variants of online/stochastic learning algorithms have been studied, many of which take the form of
\be\label{kernelOnline}
h_{t+1} =(1 - \lambda) h_t - \eta_t (h_t(x_t) - y_t) K_{x_t}, \qquad t=1, \cdots, T,
\ee
and generalization properties have been derived. Here $\{\eta_t>0\}$ is a step-size sequence, and $\lambda$ can be chosen as a positive constant depending on the sample size $\lambda(T)>0$ \cite{yao2006dynamic,tarres2014online}, or to be zero \cite{ying2008online,shamir2013stochastic,lin2017online}. In general, the computational complexities of the algorithm are $O(T+Td)$ in space and $O(T^2 d)$ in time.

According to Bochner's theorem, a continuous kernel $K(x, x') = k(x - x')$ on $\mR^d$
is positive definite if and only if $k(\delta)$ is the Fourier transform of a non-negative measure. Thus, most shift-invariant kernel functions can be expressed as an integration of some random features. A basic example for the Gaussian kernel is detailed as follows.

\begin{Exa}[Random Fourier Features \cite{rahimi2007random}]
  Let the Gaussian kernel $$K(x, x') = e^{-{\|x-x'\|^2 \over 2\sigma^2}},$$ for some $\sigma>0.$ Then according to Fourier inversion theorem, and by a simple calculation, one can prove that  \bea
  K(x,x') = {\sigma^d \over (\sqrt{2\pi})^{d+2} } \int_{\mR^d} \int_0^{2\pi}  \sqrt{2}\cos(\omega^{\top} x + b) \sqrt{2}\cos(\omega^{\top} x' + b) e^{-{\sigma^2\|\omega\|^2 \over 2} } d\omega db.
  \eea
 \end{Exa}

 Replacing $K_{x_t}$ in \eref{kernelOnline} by an unbiasd estimate with respect to a random feature, we get the doubly stochastic learning algorithms\footnote{Note that \cite{dai2014scalable} studied the algorithm with a general convex loss function. Specializing to the square loss leads to the algorithm \eref{Alg}.}.
Let $\mu$ be another probability measure on a measurable set $V$, and $\phi: V \times X \to \mR$ a  square-integrable (with respect to $\mu \otimes \rho_X$) function. Assume that the kernel $K$ can be written as \cite{rahimi2007random,bach2017equivalence}
\be \label{kernel_integration}
K(x,x') = \int_{V} \phi(v,x) \phi(v,x') d \mu(v) = \la \phi(\cdot,x), \phi(\cdot,x') \ra_{L^2_{\mu}}, \qquad \forall x,x' \in X.
\ee
Let $v_1,\cdots,v_T $ be $T$ elements in $V$, i.i.d. according to the distribution $\mu$.
The \emph{doubly stochastic learning algorithm} associated with random features $\{\phi_{v_t}\}_t$ is defined by $f_1 =0$ and
\be\label{Alg}
f_{t+1}=(1 - \eta_t \lambda)f_t- \eta_t (f_t(x_t) - y_t) \phi_{v_t}(x_t) \phi_{v_t}, \qquad t=1, \ldots, T. \ee
  The computational complexities of the algorithm are $O(T)$ (independent of the dimension of the data) in space and $O(T^2 d)$ in time.

In this paper, we study the generalization properties of Algorithm \eref{Alg}, either with a fixed constant step-size $\{\eta_t=\eta\}_{t}$ or a decaying step-size $\{\eta_t=\eta t^{-\theta}\}_{t},\theta\in(0,1)$, where $\lambda \geq 0$. Under basic assumptions in the standard learning theory and with appropriate choices of parameters, we shall prove upper bounds for the excess expected risks, i.e.,
$
\mE\|f_T - f_{\rho}\|_{\rho}^2.
$

\paragraph{Notation}
$\mN$ denotes the set of positive integers. $(a)_+ = \max(a,0)$ for any $a \in \mR.$
For $t \in \mN,$ the set $\{1,2, ..., t\}$ is denoted by $[t]$.
We will use the following conventional notations $0^0 = 1,$ $1/0 = \infty,$ $\prod_{j=t+1}^t a_j = 1$ and $\sum_{j=t+1}^t a_j = 0$ for any sequence of real numbers $\{a_j\}_{j\in \mN}.$
For any operator $L: H \to H,$
 on a Hilbert space $H$,  $I$ denotes the identity operator on $H$
 and $\Pi_{t+1}^T(L) = \prod_{k=t+1}^T (I - \eta_k L)$ when $t \in [T-1]$ and $\Pi_{T+1}^T(L) = I$.
  For a given bounded operator $L: \mcL^2_{\rho} \to \mcL^2_{\rho}, $ $\|L\|$ denotes the operator norm of $L$, i.e., $\|L\| = \sup_{f\in \mcL^2_{\rho}, \|f\|_{\rho}=1} \|Lf\|_{\rho}$. For two positive sequences $\{a_i\}_{i}$ and $\{b_i\}_{i},$ $a_i \leq O(b_i)$ (or $a_i\lesssim b_i$) stands for $a_i\leq Cb_i$ for some positive constant $C$ (independent of $i$) for all $i$. The indicator function of a subset $A$ is denoted by ${\bf 1}_{A}.$

\section{Generalization Properties for Doubly Stochastic Learning Algorithms}\label{sec:generalization}
In this section, after introducing some basic assumptions, we state our main results, following with simple discussions.
\subsection{Assumptions}
We first make the following basic assumption, with respect to the RKHS and its associated kernel as well as the underlying features.
\begin{as}\label{as:boundedness}
$\mcHK$ is separable and $K$ is measurable. Furthermore,
  there exists a positive constant $\kappa \geq 1$, such that
 $ K(x,x) \leq \kappa^2 $
 and
 $ \phi_{v}(x) \phi_{v}(x') \leq \kappa^2 $  almost surely with respect to $\mu \otimes \rho_X$.
\end{as}

The bounded assumptions on the kernel function and random features are fairly common. For example, when $K(\cdot,\cdot)$ is a Gaussian kernel with variance $1$, $K(x,x') = e^{-{\|x-x'\|^2 / 2}}$,  we have $\kappa^2 =1$.

To present our next assumption, we need to introduce
the integral operator  $L_K: \mcL^2_{\rho} \to \mcL^2_{\rho}$, defined as
\be\label{integraloper}
 L_{K} (f) =  \int_{X} f(x) K_x d\rho_{X}(x).
 \ee
 Under Assumption \ref{as:boundedness}, the operator $L_{K}$ is known to be symmetric, positive definite and trace class.
  Thus, its power $L_K^{\zeta}$  is well defined for $\zeta>0$.
Particularly, we know that \cite{cucker2007learning,steinwart2008support} $L_K^{\zeta}(\mcL^2_{\rho}) \subseteq \mcHK$ for $\zeta > {1 \over 2}$ and
$L_K^{1/2}(\mcL^2_{\rho}) = \mcHK$ with 
\be \label{isometry}
\|L^{1/2}_K g\|_K = \| g\|_{\rho}, \qquad \mbox{for all } g \in \mcL^2_{\rho}.
\ee
We make the following assumption on the regularity of the regression function.
\begin{as}\label{as:regularity}
  There exists $\zeta>0$ and $R>0$, such that $\|L_K^{-\zeta} f_{\rho}\|_{\rho} \leq R.$
\end{as}

The above assumption is very standard \cite{cucker2007learning,steinwart2008support} in nonparametric regression. It characterizes how big is  the subspace  that the target function $f_{\rho}$ lies in. Particularly, the bigger the $\zeta$ is,
the more stringent is the assumption  and the smaller is  the subspace, since $L_K^{\zeta_1}(L_{\rho}^2) \subseteq L_K^{\zeta_2}(L_{\rho}^2)$ when $\zeta_1 \geq \zeta_2.$ Moreover, when $\zeta =0,$ we are making no assumption as $\|f_{\rho}\|_{\rho}< \infty$ holds trivially, while for $\zeta = 1/2,$ we are requiring $f_{\rho} \in \mcHK$\footnote{This should be interpreted as that there exists a $f_{*} \in \mcHK$ such that $f_{\rho} = f_*$ $\rho_X$-almost surely.}.

Finally, the last assumption is related to the capacity of the RKHS.
\begin{as}\label{as:eigenvalues}
  For some $\gamma \in [0,1]$ and $c_{\gamma}>0$, $L_K$ satisfies
\be\label{eigenvalue_decay}
 \tr(L_K(L_K+\lambda I)^{-1})\leq c_{\gamma} \lambda^{-\gamma}, \quad \mbox{for all } \lambda>0.
\ee
\end{as}
The left hand-side of  \eref{eigenvalue_decay} is called as the effective
dimension \cite{caponnetto2007optimal}, or the degrees of freedom.
It can be related to covering/entropy number conditions, see \cite{steinwart2009oracle,steinwart2008support} for further details.
Assumption \ref{as:eigenvalues} is always true for $\gamma=1$ and $c_{\gamma} =\kappa^2$, since
 $L_K$ is a trace class operator which implies the eigenvalues of $L_K$, denoted as $\sigma_i$, satisfy
 $\tr(L_K) = \sum_{i} \sigma_i \leq \kappa^2.$
  The case $\gamma=1$ is referred to as the capacity independent setting.
  Assumption \ref{as:eigenvalues} with $\gamma \in]0,1]$ allows to derive better error rates. It is satisfied, e.g.,
   if the eigenvalues of $L_K$ satisfy a polynomial decaying condition $\sigma_i \sim i^{-1/\gamma}$, or with $\gamma=0$ if $L_K$ is finite rank.  
   Kernels with polynomial decaying eigenvalues include those that underlie for the Sobolev
   spaces with different orders of smoothness  (e.g. \cite{gu2002smoothing}).
   As a concrete example, the first-order Sobolev kernel $K(x, x')
   = 1 + \min\{x, x'\}$ generates a
   RKHS of Lipschitz functions, and one has that $\sigma_i \sim i^{-2}$ and thus $\gamma={1\over 2}$.
\subsection{Main Results}
We are now ready to present our main results, whose proofs are postponed to Section \ref{sec:deriving}.
Our first main result provides generalization error bounds for the studied algorithms with $\lambda=0$ and a constant (but depending on $T$)  step-size.

 \begin{thm}\label{thm:general_fixeta_un}
Under Assumptions \ref{as:boundedness}, \ref{as:regularity} and \ref{as:eigenvalues}, Let $\{f_{t}\}_{t\in [T]}$ be generated by \eref{Alg} with $\lambda=0$, $\eta_t = \eta(T)$ for all $t\in [T]$ such that
\be\label{etaT}
0< \eta(T)^{\gamma+1} T^{\gamma} \ln(\mathrm{2}T) \leq {1 \over 4 \kappa^2 (c_{\gamma} + \kappa^2)}.
\ee
 Then
\be\label{Terror_etaT}
\mE [\|f_{T+1} - f_{\rho}\|_{\rho}^2] \leq O( (\eta(T) T)^{-2\zeta}  + \eta(T)^{\gamma+1} T^{\gamma} \ln T ). \ee
{Here, the constant $C$ in the right-hand side depends only on $R,\zeta, \|f_{\rho}\|_{\rho}, \mcE(f_{\rho}),\kappa^2,c_{\gamma}$, and will be given explicitly in the proof.}
\end{thm}
According to \eref{etaT}, to derive a convergence result from the above theorem, one can  choose $\eta(T) = \eta_1 T^{-\alpha},$ with  ${\gamma \over 1+\gamma} < \alpha<1$ for some appropriate $\eta_1.$ The error bound \eref{Terror_etaT} is composed of two terms, which arise from estimating the initial and sample errors respectively in our proof, and are controlled by $\eta(T)$ directly.
A bigger $\eta(T)$ may lead to a smaller initial error but may enlarge the sample error, while a smaller $\eta(T)$ may reduce the sample error but may enlarge the initial error.   Solving this trade-off  leads to the best rate obtainable from the above theorem, which is stated  next.

\begin{corollary}\label{cor:fixedStepsCap}
\label{cor:eta}
Under Assumptions \ref{as:boundedness}, \ref{as:regularity}  and \ref{as:eigenvalues}, let $\{f_{t}\}_{t\in [T]}$ be generated by \eref{Alg} with $\lambda=0$ and
\be\label{etaTS}
 \eta_t = {\zeta \over 4\kappa^2 (c_{\gamma} + \kappa^2)(\zeta+1)} T^{-{\gamma + 2\zeta \over \gamma + 2\zeta + 1}}, \qquad \forall t \in [T].
\ee
Then,
\be
\mE [\|f_{T+1} - f_{\rho}\|_{\rho}^2] \leq O( T^{ -{ 2\zeta \over 2\zeta + \gamma + 1}} \ln T ). \ee
\end{corollary}
The above corollary asserts that with an appropriate fixed step-size, the doubly stochastic learning algorithm without regularization achieves generalization error bounds of order $O( T^{ -{ 2\zeta \over 2\zeta + \gamma + 1}} \ln T ).$

As mentioned before, Assumption \ref{as:eigenvalues} is always satisfied with $c_{\gamma} = \kappa^2$ and $\gamma=1$, which is called as the capacity independent case. Setting
$c_{\gamma} = \kappa^2$ and $\gamma=1$ in Corollary \ref{cor:fixedStepsCap}, we have the following results in the capacity independent cases.

\begin{corollary}
\label{cor:eta1}
Under Assumptions \ref{as:boundedness} and \ref{as:regularity}, let $\{f_{t}\}_{t\in [T]}$ be generated by \eref{Alg} with $\lambda=0$ and
\bea
 \eta_t = {\zeta \over 8\kappa^4 (\zeta+1)} T^{-{ 2\zeta + 1 \over 2\zeta + 2}}, \qquad \forall t \in [T].
\eea
Then,
\bea
\mE [\|f_{T+1} - f_{\rho}\|_{\rho}^2] \leq O( T^{ -{ \zeta \over \zeta + 1}} \ln T ). \eea
\end{corollary}

The above corollary can be further simplified as follows if we consider the special case $f_{\rho} \in \mcHK,$ i.e, Assumption \ref{as:regularity} with $\zeta =1/2.$

\begin{corollary}
\label{cor:eta2}
Under Assumption \ref{as:boundedness}, let $f_{\rho} \in \mcHK$ and $\{f_{t}\}_{t\in [T]}$ be generated by \eref{Alg} with $\lambda=0$ and
$
 \eta_t = 1 / (24\kappa^4 \sqrt[3]{T^2}), \forall t \in [T].
$
Then,
\bea
\mE[\|f_{T+1} - f_{\rho}\|_{\rho}^2] \leq O( T^{ -{ 1\over 3}} \ln T ). \eea
\end{corollary}

 Theorem \ref{thm:general_fixeta_un} and its corollaries provide generalization error bounds for  the studied algorithm without regularization in the fixed step-size setting. In the next theorem, we give generalization error bounds for the studied algorithm \eref{Alg} without regularization in a decaying step-size setting.

\begin{thm}
\label{thm:general_theta}
Under Assumptions \ref{as:boundedness},\ref{as:regularity} and \ref{as:eigenvalues},
 let $\gamma \neq 1,$ $\lambda=0$ and $\eta_t = \eta_1 t^{-\theta}$ for all $t \in \mN$ with ${\gamma \over \gamma+1}<\theta <1$ and $\eta_1$ such that
\be\label{eta1restriction}
0< \eta_1 \leq {1 \over 4\kappa^2 (2^{2\theta}c_{\gamma} + \kappa^2) c_{\theta,\gamma}},
\ee
where
 \be\label{cthetagamma} c_{\theta,\gamma} = \max_{t \in[T] }\left\{  t^{\gamma -\theta(\gamma + 1) + (2\theta-1)_+}
    \ln(\mathrm{2}t)\right\}. \ee
Then, for any $t\in[T],$
\be
\mE [\|f_{t+1} - f_{\rho}\|_{\rho}^2] \leq O( t^{2\zeta (\theta - 1)}  +  t^{\gamma - \theta(\gamma+1) + (2\theta - 1)_+} \ln t). \ee
\end{thm}
Similarly, there is a trade-off problem in the error bounds of the above theorem.
 Balancing the last two terms of the error bounds, we get the following corollary.
\begin{corollary}\label{cor:decayStepSize}
\label{cor:theta}
Under Assumptions \ref{as:boundedness}, \ref{as:regularity} and \ref{as:eigenvalues},  let $\gamma \neq 1,$ $\lambda=0$ and $\eta_t = \eta_1 t^{-\theta}$ for all $t \in [T]$.\\
 a) If $2\zeta < 1 - \gamma$, then by selecting $\theta = {2\zeta + \gamma \over 2\zeta + \gamma + 1}$ and $\eta_1 = {\zeta \over 3\kappa^2 (2 c_{\gamma}+ \kappa^2)},$
  \be
\mE [\|f_{t+1} - f_{\rho}\|_{\rho}^2] \leq O( t^{ -{ 2\zeta \over 2\zeta + \gamma + 1}} \ln t ). \ee
b) If $2\zeta \geq 1 - \gamma$, then by selecting $\theta = {1/2}$ and $\eta_1 = {1-\gamma \over 6\kappa^2 (2c_{\gamma} +\kappa^2)},$
  \be
\mE [\|f_{t+1} - f_{\rho}\|_{\rho}^2] \leq O( t^{\gamma -1 \over 2} \ln t ). \ee
\end{corollary}
Corollary \ref{cor:theta} asserts that with an appropriate choice of the decaying exponent for the step-size, the doubly stochastic learning algorithm without regularization has a generalization error bound of order $O( T^{ -{ 2\zeta \over 2\zeta + \gamma + 1}} \ln T )$ when $2\zeta < 1 - \gamma$, or of order $O( T^{\gamma -1 \over 2} \ln T )$ when $2\zeta \geq 1 - \gamma$.
Comparing Corollary \ref{cor:fixedStepsCap} with Corollary \ref{cor:decayStepSize},
the latter has a slower convergence rate when $2\zeta \geq 1-\gamma.$  This suggests that the fixed step-size setting may be more favourable.

Theorems \ref{thm:general_fixeta_un} and \ref{thm:general_theta} provide generalization error bounds for doubly stochastic learning algorithms without regularization. In the next theorem, we give generalization error bounds for doubly stochastic learning algorithms with regularization.

\begin{thm}
\label{thm:general_lambda}
Under Assumptions \ref{as:boundedness}, \ref{as:regularity} and \ref{as:eigenvalues}, let $\zeta \leq 1$, $\gamma \neq 1,$ $\lambda = T^{\theta-1+ \epsilon}$, $\eta_t = \eta_1 t^{-\theta}$ for all $t \in \mN$, with ${\gamma \over \gamma+1}<\theta <1$, $0< \epsilon \leq 1-\theta,$ and $\eta_1$ such that
\eref{eta1restriction}.
Then,
\be
\mE_{\{z_j,v_j\}_{j=1}^T} [\|f_{T+1} - f_{\rho}\|_{\rho}^2] \leq O(T^{2 \zeta (\theta-1 +\epsilon)}  +  T^{\gamma - \theta(\gamma+1)} \ln T). \ee
\end{thm}
Balancing the two terms from the error bounds in the above theorem to optimize the bounds, we can get the following results.
\begin{corollary}
\label{cor:lambda}
  Under Assumptions \ref{as:boundedness}, \ref{as:regularity} and \ref{as:eigenvalues}, let $\zeta \leq 1$, $\gamma \neq 1.$
For all $t\in [T],$  let $\eta_t = {\zeta \over 3\kappa^2 (3 c_{\gamma}+ \kappa^2)(1+\zeta)} t^{-{2\zeta+\gamma\over 2\zeta+\gamma+1}}$ and $\lambda = T^{-{1 \over 2\zeta+\gamma+1} + {\epsilon \over 2\zeta}},$ with $0<\epsilon \leq  {2\zeta \over 2\zeta+ \gamma+1} $. Then
\be
\mE [\|f_{T+1} - f_{\rho}\|_{\rho}^2] \leq O( T^{-{2\zeta \over 2\zeta+\gamma+1} + \epsilon } \ln T). \ee
\end{corollary}
The above corollary asserts that for some appropriate choices on the regularized parameter $\lambda$ and the decaying exponent $\theta$ of the step-size, doubly stochastic learning algorithm with regularization achieves generalization error bounds of order $O( T^{-{2\zeta \over 2\zeta+\gamma+1} + \epsilon } \ln T),$ where $\epsilon$ can be arbitrarily close to zero.
The convergence rate from Corollary \ref{cor:lambda} is essentially the same as that from Corollary \ref{cor:eta} for $\zeta \leq 1$. For the case $\zeta\geq 1$, the best obtainable rate from Corollary \ref{cor:lambda} for the studied algorithm is of order $O(T^{-{2 \over 3+\gamma} } \ln T)$. This type of phenomenon is called as saturation effect in learning theory. Note that kernel ridge regression also saturates when $\zeta>1$.


\paragraph{Discussions}  
We compare our results with those in \cite{dai2014scalable}. A regularized version of doubly stochastic learning algorithms with a convex loss function was studied in \cite{dai2014scalable}. When the loss function is the square loss, the algorithm in \cite{dai2014scalable} is exactly Algorithm \eref{Alg}. \cite[Theorem 6]{dai2014scalable} asserts that with high probability, the learning sequence generated by \eref{Alg} with $\lambda>0$ and
$\eta_t\simeq {1\over \lambda t}$,  satisfies
\be\label{eq:songlei1}
\mcE(f_{T+1}) - \mcE(f_{\lambda}) \lesssim \lambda^{-2} T^{-{1\over 2}}\ln T , \ee
provided that $\|f_t\|_{\infty}\lesssim 1$.
Here $f_{\lambda}$ is the solution of the regularized risk minimization
\bea
\min_{f \in \mcHK} \mcE(f) + \lambda\|f\|_{K}^2.
\eea
Combining \eref{eq:songlei1} with the fact that \cite{smola2000sparse} under Assumption \ref{as:regularity} with $\zeta\leq 1$,
$$\mcE(f_{\lambda}) - \mcE(f_{\rho}) \lesssim \lambda^{2\zeta}, $$
one has
\bea
\mcE(f_{T+1}) - \mcE(f_{\rho}) \lesssim \lambda^{-2} T^{-{1\over 2}}\ln T  + \lambda^{2\zeta}.
\eea
The optimal obtainable error bound is achieved by setting $\lambda_*\simeq T^{-{1\over 4\zeta+4}}$, in which case,
\bea
\mcE(f_{T+1}) - \mcE(f_{\rho}) \lesssim T^{-{\zeta\over 2\zeta+2}}\ln T.
\eea
Comparing the above result with Corollaries \ref{cor:eta1} and \ref{cor:lambda}, the error bounds (of order $O(T^{-{\zeta \over \zeta+1}})$ in the capacity independent case) from Corollaries \ref{cor:eta1} and \ref{cor:lambda} are better, while they do not require the bounded assumption $\|f_t\|_{\infty}\lesssim 1.$

We discuss some issues that might be considered in the future. First,
our generalization error bounds are in expectation, and it would be interesting to derive high-probability error bounds in the future. Second, the rates in our results are not optimal and they should be further improved in the future by using a more involved technique (perhaps with a better estimate on the sample variance). Finally, in this paper, we only consider simple stochastic gradient methods (SGM) with last iterates. It would be interesting to extend our analysis to different variants of SGM, such as the fully online/stochastic learning \cite{ye2007fully,tarres2014online}, SGM with mini-batches \cite{cotter2011better},
 the stochastic average gradient \cite{schmidt2013minimizing}, averaging SGM  \cite{dieuleveut2016harder}, multi-pass SGM \cite{lin2016optimal}, and stochastic pairwise learning \cite{ying2016stochastic}  in the future. 

\section{Error Decomposition}\label{sec:error}
The rest of this paper is devoted to proving our main results. To this end, we need some preliminary analysis and a key error decomposition.

For notational simplicity, we denote $L_K + \lambda I$ by $L_{K,\lambda}$ for any $\lambda \geq 0$, and set the residual vector
\bea
r_t = f_t - f_{\rho}, \qquad \forall t \in \mN.
\eea
Since $\{f_{t}\}_t$ is generated by (\ref{Alg}), subtracting $f_{\rho}$ from both sides of \eref{Alg},
by direct computations, one can easily prove that
\be\label{decomposition}
r_{t+1} = (I - \eta_t L_{K,\lambda}) r_t + \eta_t M_t  - \eta_t \lambda f_{\rho},
\ee
where we denote
\be\label{Mt}
M_t = L_K(f_t - f_{\rho}) - (f_t(x_t) - y_t) \phi_{v_t}(x_t) \phi_{v_t}.
\ee
 Using the iterated relationship \eref{decomposition} multiple times, we can prove the following error decomposition.
\begin{pro}\label{pro:errordecomposition}
For any $t \in [T]$, we have the following error decomposition
\be\label{errordecompostion}
\mE_{\{z_j,v_j\}_{j=1}^t} [\|r_{t+1}\|_{\rho}^2] = \|S_1(t)\|_{\rho}^2 + \mE_{\{z_j,v_j\}_{j=1}^t} [\|S_2(t)\|_{\rho}^2],
\ee
where
\be\label{S1t}
S_1(t) = \Pi_1^t(L_{K,\lambda}) f_{\rho} + \lambda \sum_{k=1}^t \eta_k \Pi_{k+1}^t(L_{K,\lambda}) f_{\rho},
\ee
and
\be\label{S2t}
S_2(t) = \sum_{k=1}^t \eta_k \Pi_{k+1}^t(L_{K,\lambda}) M_k.
\ee
\end{pro}
\begin{proof}
  Using \eref{decomposition} iteratively, with $f_1=0$ and $r_1 = f_1 - f_{\rho}$, we get
\bea
r_{t+1} = -  \left(\Pi_1^t(L_{K,\lambda}) f_{\rho} + \lambda \sum_{k=1}^t \eta_k \Pi_{k+1}^t(L_{K,\lambda}) f_{\rho}\right) + \sum_{k=1}^t \eta_k \Pi_{k+1}^t(L_{K,\lambda}) M_k,
\eea
which is exactly
\be\label{residualdecomposition}
r_{t+1} = - S_1(t) + S_2(t).
\ee
In the rest of the proof, we will write $S_i(t)$ as $S_i$ $(i=1,2)$ for short, and use the notation $\mE$ for $\mE_{\{z_j,v_j\}_{k=1}^t}.$
Following from \eref{residualdecomposition}, we get
\bea
\mE[\|r_{t+1}\|_\rho^2] = \mE[ \|-S_1 + S_2\|_{\rho}^2] = \|S_1\|_{\rho}^2 + \mE[ \| S_2\|_{\rho}^2] - 2 \mE [\la S_1, S_2\ra_{\rho}].
\eea
From \eref{Alg},  we know that for any $k\in [T]$, $f_{k+1}$ is depending only on $z_1,z_2,\cdots,z_k$ and $v_1,v_2,\cdots,v_k.$
Also, note that the family  $\{z_k,v_k\}_{k=1}^t$ is independent. Thus, we can prove that $M_k$ has the following vanishing property:
\begin{align}
\mE_{z_k,v_k} [M_k]  \overset{\eref{Mt}}{=} &  L_K(f_k - f_{\rho}) - \mE_{x_k}[(f_k(x_k) - \mE_{y_k| x_k} [y_k]) \mE_{v_k} [\phi_{v_k}(x_k) \phi_{v_k}] ] \nonumber \\
                 \overset{\eref{kernel_integration}}{=} & L_K(f_k - f_{\rho}) - \mE_{x_k}[(f_k(x_k) - \mE_{y_k|x_k} [y_k]) K_{x_k}] \nonumber \\
                 \overset{\eref{regressionfunc}}{=} & L_K(f_k - f_{\rho}) - \mE_{x_k}[(f_k(x_k) - f_{\rho}(x_k)) K_{x_k}] \nonumber \\
                 \overset{\eref{integraloper}}{=} & 0. \label{Mt=0}
\end{align}
Therefore,
\bea
\mE[\la S_1, S_2 \ra] = \sum_{k=1}^t \eta_k \la S_1,  \Pi_{k+1}^t (L_{K,\lambda})\mE[ M_k] \ra_{\rho} = 0.
\eea
The proof is complete.
\end{proof}
The error decomposition \eref{errordecompostion} is fairly common in analyzing standard stochastic/online learning algorithm \cite{ying2008online}.
The term $\|S_1(t)\|_{\rho}^2$ is related to an initial error, which is deterministic and will be estimated in the next section. The term $\mE_{\{z_j,v_j\}_{j=1}^t} [\|S_2(t)\|_{\rho}^2]$ is a sample error depending on the sample, which will be estimated in Section \ref{sec:sample}.

\section{Estimating Initial Error}\label{sec:initial}
In this section, we will upper bound the initial error, namely, the first term of the right-hand side of \eref{errordecompostion}. To this end, we introduce the following two lemmas.

\begin{lemma}\label{lemma:initialerror2}
Let $\lambda \geq 0$, $\zeta \geq 0,$ and $\eta_k$ be such that $ 0\leq \eta_k (\lambda + \kappa^2) \leq 1$ for all $k \in \mN$. Then for all $t \in \mN,$
  \be\label{initialerror2}
  \lambda \left\| \sum_{k=1}^t   \eta_k  \Pi_{k+1}^t(L_{K,\lambda}) L_K^{\zeta}\right\|  \leq \lambda^{\min(\zeta,1)} \kappa^{2(\zeta-1)_+} \mathbf{1}_{\{\lambda > 0\}}.
  \ee
\end{lemma}
\begin{proof}
\eref{initialerror2} holds trivially for the case $\lambda = 0.$  Now, we consider the case $\lambda>0.$
  Recall that $L_K$ is a self-adjoint, compact, and positive operator on $\mcL^2_{\rho}$. According to the spectral theorem, $L_K$ has only non-negative singular values $\{\sigma_i\}_{i=1}^{\infty}$ such that $\kappa^2 \geq \sigma_1 \geq \sigma_2 \geq \cdots \geq 0$.
  Thus,
  \bea
   \lambda \left\| \sum_{k=1}^t \eta_k  \Pi_{k+1}^t(L_{K,\lambda}) L_K^{\zeta}\right\|  = \lambda \sup_{i} \sigma_i^{\zeta} \sum_{k=1}^t  \eta_k \prod_{j=k+1}^t (1 - \eta_j (\lambda + \sigma_i) )  .
  \eea
 Letting $c_i = \lambda + \sigma_i$ for each $i$, we have
  \begin{align*}
    &  (\lambda + \sigma_i) \sum_{k=1}^t  \eta_k \prod_{j=k+1}^t (1 - \eta_j (\lambda + \sigma_i) ) \\
   =&  \sum_{k=1}^t (1 - (1 - \eta_kc_i)) \prod_{j=k+1}^t (1 - \eta_j c_i ) \\
   =&  \sum_{k=1}^t \left(\prod_{j=k+1}^t (1 - \eta_j c_i ) - \prod_{j=k}^t (1 - \eta_j c_i )\right)\\
    =& \left( 1 -  \prod_{j=1}^t (1 - \eta_j c_i ) \right) \\
    \leq& 1.
  \end{align*}
  Therefore, we get
  \bea
  \lambda\left\| \sum_{k=1}^t \eta_k  \Pi_{k+1}^t(L_{K,\lambda}) L_K^{\zeta}\right\|  \leq \sup_{i} { \lambda \sigma_i^{\zeta} \over \lambda + \sigma_i}.
  \eea
  When $\zeta \in [0,1],$  we have
  $${\lambda \sigma_i^{\zeta} \over \lambda +\sigma_i} = \lambda^{\zeta} \left({\lambda \over \lambda + \sigma_i}\right)^{1-\zeta} \left({\sigma_i \over \lambda + \sigma_i}\right)^{\zeta} \leq \lambda^{\zeta}. $$
  When $\zeta >1,$
  \bea
  {\lambda \sigma_i^{\zeta} \over \lambda +\sigma_i} \leq \lambda \sigma_i^{\zeta-1} \leq \lambda \kappa^{2(\zeta-1)}.
  \eea
  From the above analysis, we can get \eref{initialerror2}.
  The proof is complete.
\end{proof}
\begin{lemma}
  \label{lemma:initialerror1}
  Under the assumptions of Lemma \ref{lemma:initialerror2}, we have for $t\in \mN$ and any non-negative integer $k  \leq t - 1,$
  \be\label{initialerror1_interm}
  \| \Pi_{k+1}^t(L_{K,\lambda}) L_K^{\zeta}\| \leq \exp\left\{ - \lambda  \sum_{j=k+1}^t \eta_j\right\} \left( \zeta \over e \sum_{j=k+1}^t \eta_j \right)^{\zeta}.
  \ee
\end{lemma}
The above lemma is essentially proved in \cite{ying2008online,tarres2014online}. For completeness, we provide a proof in the appendix.

Now, we can upper bound the initial error as follows.
\begin{pro}\label{pro:initialerror}
Under Assumption \ref{as:regularity}, let $\eta_t = \eta_1 t^{-\theta}$ for all $t \in \mN$,  with $\eta_1>0$ such that $\eta_1 (\lambda + \kappa^2) \leq 1$ and $\theta \in [0,1].$ Then, for any $t \in \mN,$
\be\label{initialerror}
\|S_1(t)\|_{\rho} \leq R \kappa^{2(\zeta-1)_+} \lambda^{\min(\zeta,1)} +
\left( \zeta \over \eta_1  \right)^{\zeta} R \cdot \left\{
\begin{array}
    {ll}
    \exp\left\{ - \lambda  \eta_1 t^{1-\theta}/2 \right\}   t^{\zeta (\theta - 1)}, & \hbox{when} \ \theta \neq 1, \\
   t^{- \eta_1 \lambda} \left\{ \ln (t+1) \right\}^{-\zeta}, & \hbox{when} \ \theta = 1,
  \end{array}
\right.
\ee
and
\be\label{initialerrormax}
\|S_1(t)\|_{\rho} \leq 2 \|f_{\rho}\|_{\rho}.
\ee
\end{pro}

\begin{proof}
 Note that $S_1(t)$ is given by \eref{S1t}. Thus, we have
  \be\label{eq4}
  \|S_1(t)\|_{\rho} \leq \left\| \lambda \sum_{k=1}^t \eta_k \Pi_{k+1}^t(L_{K,\lambda}) f_{\rho} \right\|_{\rho} + \| \Pi_1^t(L_{K,\lambda}) f_{\rho}\|_{\rho} .
  \ee
 With Assumption \ref{as:regularity}, we can write $f_{\rho} = L_K^{\zeta} g_0$ for some $\|g_0\|_\rho \leq R$. We thus derive
   \bea
   \|S_1(t)\|_{\rho} \leq  R \lambda \left\| \sum_{k=1}^t \eta_k  \Pi_{k+1}^t(L_{K,\lambda}) L_K^{\zeta}\right\|  +  R \| \Pi_1^t(L_{K,\lambda}) L_K^{\zeta}\| .
  \eea
 Note that $\eta_k = \eta_1 k^{-\theta}$ with $\eta_1>0$ satisfying $\eta_1 (\lambda + \kappa^2)\leq 1$ and $\theta \in [0,1]$ implies that $ 0\leq \eta_k (\lambda + \kappa^2) \leq 1$ for all $k \in \mN.$
 Thus, we can use \eref{initialerror2} and \eref{initialerror1_interm} to bound the last two terms and get that
\be\label{initialerror_interm}
   \|S_1(t)\|_{\rho} \leq  R \kappa^{2(\zeta-1)_+} \lambda^{\min(\zeta,1)}  +  R \exp\left\{ - \lambda  \sum_{k=1}^t \eta_k\right\} \left( \zeta \over e \sum_{k=1}^t \eta_k \right)^{\zeta} .
  \ee
   Observe that
  \be
  \sum_{k=1}^t k^{-\theta} \geq \sum_{k=1}^t \int_k^{k+1} x^{-\theta} dx =
  \left\{ \begin{array}
    {ll}
     {(t+1)^{1-\theta} - 1 \over 1 -\theta}, & \hbox{when} \ \theta \in [0,1),\\
    \ln (t + 1),                           & \hbox{when} \ \theta = 1,
  \end{array}
  \right.
  \ee
  and that by the mean value theorem,
  \bea
  {(t+1)^{1-\theta} - 1 \over 1 -\theta} \geq  {t (1 -\theta)(t+1)^{-\theta} \over 1-\theta} \geq {t^{1-\theta} \over 2}.
  \eea
  We thus have
  \be
  \sum_{k=1}^t \eta_k = \eta_1 \sum_{k=1}^t k^{-\theta} \geq \left\{ \begin{array}
    {ll}
    \eta_1  t^{1-\theta}/2,  & \hbox{if} \ \theta \in [0,1),\\
    \eta_1 \ln (t+1),                         & \hbox{if} \ \theta = 1,
  \end{array}
  \right.
  \ee
  and consequently,
  \begin{align*}
  \exp\left\{ - \lambda  \sum_{k=1}^t \eta_k\right\} \left( \zeta \over e \sum_{k=1}^t \eta_k \right)^{\zeta}
  \leq \left( \zeta \over  \eta_1  \right)^{\zeta} \cdot
  \left\{\begin{array}
    {ll}
    \exp\left\{ - \lambda  \eta_1 t^{1-\theta}/2 \right\}   t^{(\theta - 1)\zeta}, & \hbox{if} \ \theta \in [0,1), \\
   t^{-\eta_1 \lambda} \left\{ \ln (t+1) \right\}^{-\zeta}, & \hbox{if} \ \theta = 1.
  \end{array}
  \right.
   \end{align*}
   Putting the above inequality into \eref{initialerror_interm}, we get the desired bound \eref{initialerror}.

Besides, by \eref{eq4}, we also have
\be\label{eq5}
\|S_1(t)\|_{\rho} \leq \| \Pi_1^t(L_{K,\lambda})\| \| f_{\rho}\|_{\rho}  + \left\| \lambda \sum_{k=1}^t \eta_k \Pi_{k+1}^t(L_{K,\lambda}) \right\| \| f_{\rho}\|_{\rho} .
\ee
Since
\bea
\| \Pi_1^t(L_{K,\lambda})\| = \sup_{i} \prod_{j=1}^t (1 - \eta_k(\lambda + \sigma_i)) \leq 1,
\eea
and by setting $\zeta = 0$ in \eref{initialerror2},
\bea
\left\| \lambda \sum_{k=1}^t \eta_k \Pi_{k+1}^t(L_{K,\lambda}) \right\| \leq \mathbf{1}_{\{\lambda>0\}}.
\eea
Introducing the last two inequalities into \eref{eq5}, we get the desired bound \eref{initialerrormax}.
The proof is complete.
\end{proof}

\section{Bounding Sample Error}\label{sec:sample}
In this section, we will upper bound the sample error, i.e., the last term of \eref{errordecompostion}. We first introduce the following decomposition.
\begin{pro}\label{pro:serrors_interm}
Under Assumption \ref{as:boundedness},
let $\{f_{t+1}\}_{t=1}^T$ be generated by Algorithm \eref{Alg}, $S_2(t)$ be given by \eref{S2t}, with $M_t$  given by \eref{Mt}.
 Then for any $t \in \mN,$
\be\label{serrors_interm}
\mE_{\{z_j,v_j\}_{j=1}^t}  \left\{ \|S_2(t)\|_{\rho}^2 \right\}
\leq \kappa^2 \sum_{k=1}^t  \eta_k^2   \mE_{v_k} [\|\Pi_{k+1}^t(L_{K,\lambda}) \phi_{v_k} \|_{\rho}^2] \mE_{\{z_j,v_j\}_{j=1}^{k-1}} [\mcE(f_k)].
\ee
\end{pro}

\begin{proof}
As in the proof of Proposition \ref{pro:errordecomposition}, we will use the notation $\mE$ for $\mE_{\{z_j,v_j\}_{j=1}^t}.$
Note that $S_2(t)$ is given by \eref{S2t}. Thus, we have
\bea
\mE[\|S_2(t)\|_{\rho}^2] = \sum_{k,l=1}^t \eta_k \eta_l \mE \la \Pi_{k+1}^t(L_{K,\lambda}) M_k , \Pi_{l+1}^t(L_{K,\lambda}) M_l \ra_{\rho}.
\eea
When $k\neq l$, without loss of generality, we can assume that $k>l.$
Recall that $M_l$ is given by \eref{Mt}, and $M_l$ is depending only on $\{z_j,v_j\}_{j=1}^{l}$.
Thus, we have
\begin{align*}
&\mE \la \Pi_{k+1}^t(L_{K,\lambda}) M_k , \Pi_{l+1}^t(L_{K,\lambda}) M_l \ra_{\rho} \\
=& \mE_{\{z_j,v_j\}_{j=1}^{k-1}} \la \Pi_{k+1}^t(L_{K,\lambda}) \mE_{z_k,v_k}[M_k] ,  \Pi_{l+1}^t(L_{K,\lambda}) M_l \ra_{\rho} \\
=& 0,
\end{align*}
where we have used the vanishing property \eref{Mt=0} for the last equality.
We thus get
\bea
\mE[\|S_2(t)\|_{\rho}^2] = \sum_{k=1}^t \eta_k^2 \mE \| \Pi_{k+1}^t(L_{K,\lambda}) M_k \|_{\rho}^2.
\eea
Using the fact that for any random variable $\xi \in L_{\rho}^2,$ $$\mE\|\xi - \mE[\xi] \|_\rho^2 = \mE\|\xi\|_{\rho}^2 - \|\mE[\xi] \|_\rho^2 \leq \mE\|\xi\|_{\rho}^2,$$ with $\xi = \Pi_{k+1}^t(L_{K,\lambda}) (f_k(x_k) - y_k) \phi_{v_k}(x_k) \phi_{v_k},$ we get
\bea
\mE[\|S_2(t)\|_{\rho}^2] \leq  \sum_{k=1}^t \eta_k^2 \mE\left\{ (f_k(x_k) - y_k)^2 \phi_{v_k}^2(x_k) \|\Pi_{k+1}^t(L_{K,\lambda}) \phi_{v_k} \|_{\rho}^2 \right\}.
\eea
By Assumption \ref{as:boundedness}, $\phi_{v_k}^2(x_k) \leq \kappa^2$ almost surely. And note that $f_k$ is depending only on $\{z_j,v_j\}_{j=1}^{k-1}$.
We thus can relax the above inequality as \eref{serrors_interm}.
 The proof is complete.
\end{proof}
Based on the above proposition and using an inducted argument, one can prove the following result.
\begin{pro}\label{pro:totalerror_inter}
  Instate the assumptions and notations of Proposition \ref{pro:serrors_interm}. Assume that \eref{initialerrormax} and that
\be\label{restriction}
 \kappa^2 \max_{t \in [T]} \left\{ \sum_{k=1}^t  \eta_k^2   \mE_{v_k} [\|\Pi_{k+1}^t(L_{K,\lambda}) \phi_{v_k} \|_{\rho}^2] \right\} \leq 1/2.
\ee
Then, for any $t \in [T]$,
\be\label{totalerror_interm}
\mE_{\{z_j,v_j\}_{j=1}^t} [\|r_{t+1}\|_{\rho}^2]
\leq  \|S_1(t)\|_{\rho}^2  + \left(8\|f_{\rho}\|_{\rho}^2 + 2\mcE(f_{\rho})\right) \kappa^2 \sum_{k=1}^t  \eta_k^2   \mE_{v_k} [\|\Pi_{k+1}^t(L_{K,\lambda}) \phi_{v_k} \|_{\rho}^2] .
\ee
\end{pro}

\begin{proof}
  By Proposition \ref{pro:serrors_interm}, we have \eref{serrors_interm}.
 Plugging  with \eref{excesserror}, we get
 \bea
 \begin{split}
&\mE_{\{z_j,v_j\}_{j=1}^t}  \left\{ \|S_2(t)\|_{\rho}^2 \right\} \\
 \leq&  \left(\sup_{ k\in [t]} \mE_{\{z_j,v_j\}_{j=1}^{k-1}} [\|r_k\|_{\rho}^2] + \mcE(f_{\rho})\right) \kappa^2 \sum_{k=1}^t  \eta_k^2   \mE_{v_k} [\|\Pi_{k+1}^t(L_{K,\lambda}) \phi_{v_k} \|_{\rho}^2].
\end{split}
\eea
  Introducing the above into the error decomposition \eref{errordecompostion}, we have
\be\label{eq6}
\begin{split}
&\mE_{\{z_j,v_j\}_{j=1}^t} [\|r_{t+1}\|_{\rho}^2] \\
\leq & \|S_1(t)\|_{\rho}^2  + \left(\sup_{ k\in [t]} \mE_{\{z_j,v_j\}_{j=1}^{k-1}} [\|r_k\|_{\rho}^2] + \mcE(f_{\rho})\right) \kappa^2 \sum_{k=1}^t  \eta_k^2   \mE_{v_k} [\|\Pi_{k+1}^t(L_{K,\lambda}) \phi_{v_k} \|_{\rho}^2] .
\end{split}\ee
Letting $t=1$ in the above inequality, with $r_1= f_1 - f_{\rho} = - f_{\rho}$,
\bea
\begin{split}
\mE_{z_1,v_1} [\|r_{2}\|_{\rho}^2] \leq  \|S_1(1)\|_{\rho}^2  + \left(\|f_{\rho}\|_{\rho}^2 + \mcE(f_{\rho})\right) \kappa^2 \eta_1^2   \mE_{v_1} [\|\phi_{v_1} \|_{\rho}^2].
\end{split}\eea
This verifies \eref{totalerror_interm} for $t=1$. Now for any fixed $t \in [T]$ with $t\geq 2,$ assume that \eref{totalerror_interm} holds for each $t' \in [t-1].$
In this case,
\begin{align*}
&\sup_{ k\in [t]} \mE_{\{z_j,v_j\}_{j=1}^{k-1}}[\|r_k\|_{\rho}^2] \\
 \leq& \sup_{k \in [T]} [\|S_1(k)\|_{\rho}^2]  + \left(8\|f_{\rho}\|_{\rho}^2 + 2\mcE(f_{\rho})\right) \sup_{k\in [T]} \left\{ \kappa^2 \sum_{j=1}^k  \eta_j^2   \mE_{v_j} [\|\Pi_{j+1}^k (L_{K,\lambda}) \phi_{v_j} \|_{\rho}^2] \right\}\\
\leq & 4\|f_{\rho}\|_{\rho}^2 + \left(8\|f_{\rho}\|_{\rho}^2 + 2\mcE(f_{\rho})\right) /2 \\
 =& 8\|f_{\rho}\|_{\rho}^2 + \mcE(f_{\rho}),
\end{align*}
where for the last inequality, we used \eref{initialerrormax} and \eref{restriction}.
Therefore, using \eref{eq6}, we get
\bea
\mE_{\{z_j,v_j\}_{j=1}^t}[\|r_{t+1}\|_{\rho}^2] \leq  \|S_1(t)\|_{\rho}^2  + \left(8\|f_{\rho}\|_{\rho}^2 + 2\mcE(f_{\rho})\right) \kappa^2 \sum_{k=1}^t  \eta_k^2   \mE_{v_k} [\|\Pi_{k+1}^t(L_{K,\lambda}) \phi_{v_k} \|_{\rho}^2],
\eea
which verifies the case $t.$ Thus, by an inducted argument, we prove the desired results.
\end{proof}
From the above result, we see that the sample error is upper bounded in terms of  $$\sum_{k=1}^t  \eta_k^2   \mE_{v_k} [\|\Pi_{k+1}^t(L_{K,\lambda}) \phi_{v_k} \|_{\rho}^2],$$ provided that \eref{restriction} holds.
We thus only need to estimate $\sum_{k=1}^t  \eta_k^2   \mE_{v_k} [\|\Pi_{k+1}^t(L_{K,\lambda}) \phi_{v_k} \|_{\rho}^2]$. To do so,
we  introduce the following three lemmas.
The trace of any trace operator $L : \mcL_{\rho}^2 \to \mcL_{\rho}^2 $ is denoted by $\tr(L).$ 

\begin{lemma}\label{lemma:serrors_general}
 We have for any $t \in \mN$ and any $k \in [t],$
\be\label{serrors_general}
  \mE_{v_k} [\|\Pi_{k+1}^t(L_{K,\lambda}) \phi_{v_k} \|_{\rho}^2] =  \tr (\Pi_{k+1}^t(L_{K,\lambda})^2 L_K ).
\ee
\end{lemma}

\begin{proof}
Since
\begin{align*}
\mE_{v_k}[\|\Pi_{k+1}^t(L_{K,\lambda}) \phi_{v_k} \|_{\rho}^2] =& \mE_{v_k} \tr (\Pi_{k+1}^t(L_{K,\lambda}) \phi_{v_k} \otimes \phi_{v_k} \Pi_{k+1}^t(L_{K,\lambda})) \\
=&  \tr (\Pi_{k+1}^t(L_{K,\lambda}) \mE_{v_k} [\phi_{v_k} \otimes \phi_{v_k}] \Pi_{k+1}^t(L_{K,\lambda})),
\end{align*}
and  for any $f,g \in L^2_{\rho},$ (see, e.g., \cite{bach2017equivalence})
\begin{align*}
&\la \mE_{v}[\phi_v \otimes \phi_v] f, g \ra_{\rho} = \mE_v \la \phi_v,f \ra_{\rho} \la \phi_v, g\ra_{\rho} \\
=& \int_V \int_X \int_X f(x) g(t) \phi_v(x) \phi_v(t) d \rho_X(x) d \rho_X(t) d \mu(v)  \\
=& \int_X \int_X f(x) g(t) K(x,t) d \rho_X(x) d \rho_{X}(t) \\
=& \la L_K f, g \ra_{\rho},
\end{align*}
where for the third equality, we used \eref{kernel_integration}.
Therefore,
$$
\mE_{v_k}[\|\Pi_{k+1}^t(L_{K,\lambda}) \phi_{v_k} \|_{\rho}^2] =  \tr (\Pi_{k+1}^t(L_{K,\lambda}) L_K \Pi_{k+1}^t(L_{K,\lambda})),
$$
which leads to the desired result \eref{serrors_general}. The proof is complete.
\end{proof}

\begin{lemma}\label{lemma:serrors_general_interm1}
Under Assumptions \ref{as:boundedness} and \ref{as:eigenvalues}, let $\lambda \geq 0$, and $\eta_k \in \mR^+$ be such that $\eta_k (\lambda + \kappa^2) \leq 1$ for all $t \in \mN$.
Then for any $k, t \in \mN$ with $k \leq t-1$, we have
\be\label{serrors_general_interm1}
  \tr (\Pi_{k+1}^t(L_{K,\lambda})^2 L_K ) \leq 2 c_{\gamma}  \exp\left\{ - 2 \lambda  \sum_{j=k+1}^t \eta_j\right\} \left( 2e \sum_{j=k+1}^t \eta_j  \right)^{\gamma-1}.
\ee
\end{lemma}
\begin{proof}
 Recall that $L_K$ is a self-adjoint, compact, positive operator on $\mcL^2_{\rho}$, and $L_K$ has only non-negative singular values $\{\sigma_i\}_{i=1}^{\infty}$ such that $\kappa^2 \geq \sigma_1 \geq \sigma_2 \geq \cdots \geq 0.$
  Fix $ k \in [t-1].$
For any $\lambda_0>0,$
\begin{align*}
& \tr (\Pi_{k+1}^t(L_{K,\lambda})^2 L_K ) \\
=& \tr ((L_K+\lambda_0) \Pi_{k+1}^t( L_{K,\lambda})^2 L_K (L_K+\lambda_0)^{-1}) \\
\leq& \|(L_K+\lambda_0) \Pi_{k+1}^t( L_{K,\lambda})^2\| \tr ( L_K (L_K+\lambda_0)^{-1})\\
\leq& \|(L_K+\lambda_0) \Pi_{k+1}^t( L_{K,\lambda})^2\| c_{\gamma} \lambda_0^{-\gamma},
\end{align*}
where for the last inequality we used Assumption \ref{as:eigenvalues}.
Note that by Lemma \ref{lemma:initialerror1}, we have
  \bea
  \|\Pi_{k+1}^t(L_{K,\lambda})^2 L_K\| = \|\Pi_{k+1}^t(L_{K,\lambda}) L_K^{1/2}\|^2 \leq \exp\left\{ - 2 \lambda  \sum_{j=k+1}^t \eta_j\right\} \left( 1 \over 2 e \sum_{j=k+1}^t \eta_j \right),
  \eea
  and
  \bea
   \|\Pi_{k+1}^t(L_{K,\lambda})^2\| \leq \exp\left\{ - 2\lambda  \sum_{j=k+1}^t \eta_j\right\}.
  \eea
Therefore, we get
\bea
 \tr (\Pi_{k+1}^t(L_{K,\lambda})^2 L_K ) 
\leq c_{\gamma} \exp\left\{ - 2\lambda  \sum_{j=k+1}^t \eta_j\right\} \left( { 1 \over 2 e \sum_{j=k+1}^t \eta_j } + \lambda_0 \right)  \lambda_0^{-\gamma}.
\eea
Choosing $\lambda_0 = {1 \over 2e \sum_{j=k+1}^t \eta_j},$ we can get the desired result. The proof is complete.
\end{proof}

\begin{lemma}\label{lemma:serrors_general_interm2}
Let $c \geq 0, \gamma \in [0,1]$ and $\theta \in [0,1].$ Then for any $t\in \mN$ with $t\geq 2,$
  \be
  \begin{split}
  &\sum_{k=1}^{t-1} k^{-2\theta} \exp\left\{ - c  \sum_{j=k+1}^t j^{-\theta} \right\} \left( \sum_{j=k+1}^t j^{-\theta}  \right)^{\gamma-1} \\
 \leq&  2^{2\theta-\gamma} t^{\gamma -\theta(\gamma + 1)} \ln(\mathrm{2} t) \left(  \exp\left\{ - c t^{1-\theta}/2 \right\} t^{(2\theta-1)_+}
  +   \min\left(  1, \left({ t^{1 - \theta} c}\right)^{-\gamma} \right) \right).
  \end{split}
  \ee
\end{lemma}

\begin{proof}
For any $j \in [k+1,t],$ we have $j^{-\theta} \geq t^{-\theta}.$ Thus,
\begin{align*}
& \sum_{k=1}^{t-1} k^{-2\theta} \exp\left\{ - c  \sum_{j=k+1}^t j^{-\theta} \right\} \left( \sum_{j=k+1}^t j^{-\theta}  \right)^{\gamma-1}  \\
 \leq& \sum_{k=1}^{t-1} k^{-2\theta} \exp\left\{ - c (t-k) t^{-\theta}  \right\} \left( (t-k) t^{-\theta} \right)^{\gamma - 1}.
\end{align*}
For $k \leq (t-1)/2$, we have $
  (t-k) t^{-\theta} \geq {t^{1-\theta} / 2}.
  $
Therefore,
\begin{align*}
&  \sum_{k\leq (t-1)/2} k^{-2\theta} \exp\left\{ - c  \sum_{j=k+1}^t j^{-\theta} \right\} \left( \sum_{j=k+1}^t j^{-\theta}  \right)^{\gamma-1} \\
\leq&  \exp\left\{ - c t^{1-\theta}/2 \right\} 2^{1-\gamma} t^{(1-\theta)(\gamma - 1)} \sum_{k\leq (t-1)/2}  k^{-2\theta}.
\end{align*}
Applying
 \be\label{BasicSum} \sum_{k=1}^{t} k^{-\theta'} \leq
 t^{\max(1 - \theta',0)}\sum_{k=1}^{t} k^{-1} \leq t^{(1 - \theta')_+} { \ln (et)}
 \ee
to bound $\sum_{k\leq (t-1)/2}  k^{-2\theta}$, we get
\begin{align*}
& \sum_{k\leq (t-1)/2} j^{-2\theta} \exp\left\{ - c  \sum_{j=k+1}^t j^{-\theta} \right\} \left( \sum_{j=k+1}^t j^{-\theta}  \right)^{\gamma-1} \\
\leq& 2^{1-\gamma -(1-2\theta)_+}\exp\left\{ - c t^{1-\theta}/2 \right\}  t^{(1-\theta)(\gamma - 1) + (1-2\theta)_+} \ln(e t/2) \\
\leq& 2^{\min(1,2\theta)-\gamma} \exp\left\{ - c t^{1-\theta}/2 \right\} t^{\gamma -\theta(\gamma + 1)} t^{(2\theta-1)_+} \ln (\mathrm{2}t).
\end{align*}
For $t/2 \leq k \leq t-1,$ $k^{-2\theta} \leq 2^{2\theta} t^{-2\theta}$. We thus have
\begin{align*}
& \sum_{k\geq t/2}^{t-1}  k^{-2\theta}  \exp\left\{ - c (t-k) t^{-\theta}  \right\} \left( (t-k) t^{-\theta} \right)^{\gamma - 1} \\
\leq&  2^{2\theta} t^{-2\theta} \sum_{k\geq t/2}^{t-1}  \exp\left\{ - c (t-k) t^{-\theta}  \right\} \left( (t-k) t^{-\theta} \right)^{\gamma - 1} \\
=& 2^{2\theta} t^{-\theta(\gamma + 1)} \sum_{1\leq k \leq t/2}  \exp\left\{ - c t^{-\theta} k  \right\}  k^{\gamma - 1}.
\end{align*}
On the one hand, for any $c \geq 0$ and $\gamma \geq 0$, by \eref{BasicSum},
\bea
 \sum_{1\leq k \leq t/2} \exp\left\{ - c t^{-\theta} k  \right\}   k^{\gamma - 1}
\leq \sum_{1\leq k \leq t/2}   k^{\gamma - 1} \leq 2^{-\gamma}  t^{\gamma} \ln(et/2) \leq 2^{-\gamma}  t^{\gamma} \ln(\mathrm{2}t).
\eea
On the other hand,
when $c > 0$ and $\gamma>0,$ we subsequently apply \eref{exppoly} (see the appendix) with $x = t^{-\theta} k,$ $\zeta = \gamma$, and \eref{BasicSum} to get
\bea
\sum_{1\leq k \leq t/2}  \exp\left\{ - c t^{-\theta} k  \right\}  k^{\gamma - 1} \leq \left({\gamma \over ec}\right)^{\gamma} t^{\theta \gamma} \sum_{1\leq k \leq t/2} k^{ - 1} \leq 2 ^{-\gamma} c^{-\gamma} t^{\theta \gamma} \ln(\mathrm{2} t) .
\eea
Note that the above inequality also holds for $c = 0,\gamma \geq 0,$ or $c>0, \gamma=0$, as we used the conventional notations that $0^0 = 1$ and $1/0 = \infty.$
Consequently, we derive
\begin{align*}
& \sum_{k\geq t/2}^{t-1}  k^{-2\theta}  \exp\left\{ - c (t-k) t^{-\theta}  \right\} \left( (t-k) t^{-\theta} \right)^{\gamma - 1} \\
\leq&  2^{2\theta - \gamma} t^{-\theta(\gamma + 1)} \ln(\mathrm{2} t) \min\left( t^{\gamma}, c^{-\gamma} t^{\theta \gamma} \right)\\
=& 2^{2\theta -\gamma} t^{\gamma -\theta(\gamma + 1)} \ln(\mathrm{2} t) \min\left(  1, \left({ t^{1 - \theta} c}\right)^{-\gamma} \right).
\end{align*}
From the above analysis, one can conclude the proof.
\end{proof}
We now can estimate the term related to sample error as follows.
\begin{pro}\label{pro:TraceSumEst}
Under Assumptions \ref{as:boundedness} and \ref{as:eigenvalues}, let $\eta_k  = \eta_1 k^{-\theta}$ for all $k \in \mN$,
with $0<\eta_1 \leq 1/(\lambda + \kappa^2),$ $\theta \in [0,1]$
and $\lambda \geq 0.$  Then, for all $t \in \mN,$
 \be\label{traceSumEst}
 \sum_{k=1}^{t}  \eta_k^2   \mE_{v_k} [\|\Pi_{k+1}^t(L_{K,\lambda}) \phi_{v_k} \|_{\rho}^2] \leq (2^{2\theta +\gamma -1}c_{\gamma} + \kappa^2) \mcF_{\eta_1,\theta, \lambda,\gamma}(t),
 \ee
 where
 \be\label{mcF}
 \begin{split}
  \mcF_{\eta_1,\theta, \lambda,\gamma}(t) =
  \eta_1^{\gamma+1} t^{\gamma -\theta(\gamma + 1)} \ln(\mathrm{2} t) \left(  \exp\left\{ - \lambda \eta_1 t^{1-\theta} \right\} t^{(2\theta-1)_+}
  +   1 \right).
  \end{split}
 \ee
\end{pro}
\begin{proof}
Following from Lemma \ref{lemma:serrors_general} and $\tr(L_K) \leq \kappa^2$, we have
\be\label{eq7}
\sum_{k=1}^{t}  \eta_k^2   \mE_{v_k} [\|\Pi_{k+1}^t(L_{K,\lambda}) \phi_{v_k} \|_{\rho}^2] \leq  \sum_{k=1}^{t-1}  \eta_k^2 \tr (\Pi_{k+1}^t(L_{K,\lambda})^2 L_K ) + \eta_t^2 \kappa^2.
\ee
Note that $\eta_j = \eta_1 j^{-\theta}$ implies $\eta_j (\lambda + \sigma_i) \leq \eta_1(\lambda + \kappa^2) \leq 1$ for all $j\in \mN.$
Applying Lemma \ref{lemma:serrors_general_interm1}, we get
\begin{align*}
& \sum_{k=1}^{t-1}  \eta_k^2   \mE_{v_k} [\|\Pi_{k+1}^t(L_{K,\lambda}) \phi_{v_k} \|_{\rho}^2] \\
\leq&  2 c_{\gamma} \sum_{k=1}^{t-1}  \eta_k^2   \exp\left\{ - 2 \lambda  \sum_{j=k+1}^t \eta_j\right\} \left( 2e \sum_{j=k+1}^t \eta_j  \right)^{\gamma-1}\\
\leq & 2^{2\gamma -1} c_{\gamma} \eta_1^{\gamma+1} \sum_{k=1}^{t-1} k^{-2\theta}     \exp\left\{ - 2 \eta_1 \lambda  \sum_{j=k+1}^t j^{-\theta}\right\} \left( \sum_{j=k+1}^t j^{-\theta}  \right)^{\gamma-1}.
\end{align*}
 Using Lemma \ref{lemma:serrors_general_interm2} with $c = 2\eta_1 \lambda,$ by a direct computation, we get
\bea
\sum_{k=1}^{t-1}  \eta_k^2   \mE_{v_k} [\|\Pi_{k+1}^t(L_{K,\lambda}) \phi_{v_k} \|_{\rho}^2] \leq 2^{2\theta + \gamma - 1} c_{\gamma} \mcF_{\eta_1,\theta, \lambda,\gamma}(t).
\eea
Introducing the above inequality into \eref{eq7}, and by using the fact that since $\eta_1 \leq 1/(\lambda + \kappa^2)$, $\kappa^2\geq 1$ and $\gamma,\theta \in[0,1]$,  
\bea
\eta_t^2 = \eta_1^2 t^{-2\theta} =
\eta_1^{1+\gamma} t^{\gamma - \theta(\gamma+1)} \eta_1^{1 - \gamma} t^{(\theta-1)\gamma - \theta} \leq \eta_1^{1+\gamma} t^{\gamma - \theta(\gamma+1)}  \leq \mcF_{\eta_1,\theta, \lambda,\gamma}(t),
\eea
one can get the desired result.
The proof is finished.
\end{proof}
\section{Deriving Total Error}\label{sec:deriving}
In this section, we estimate the total errors for the studied algorithms with different choices of parameters.

\subsection{Case 1: $\lambda=0,\eta_1 = \eta(T), \theta=0$}

\begin{proof}[Proof of Theorem \ref{thm:general_fixeta_un}]
By Proposition \ref{pro:TraceSumEst}, we have \eref{traceSumEst}, where $\mcF_{\eta_1,\theta, \lambda,\gamma}$ is given by \eref{mcF}. Plugging with $\lambda=0$, $\theta = 0$, $\eta_1 = \eta(T)$, and using $2^{\gamma -1} \leq 1$ since $\gamma\in [0,1]$, we have that for any $t\in [T],$
\be\label{eq1}
\sum_{k=1}^{t}  \eta_k^2   \mE_{v_k} [\|\Pi_{k+1}^t(L_{K,\lambda}) \phi_{v_k} \|_{\rho}^2] \leq (c_{\gamma} + \kappa^2)2 \eta(T)^{\gamma+1} t^{\gamma} \ln(\mathrm{2}t).
\ee
Taking the maximum over $t \in [T],$  and then multiplying both sides by $\kappa^2,$
\bea
\kappa^2 \max_{t \in [T]} \sum_{k=1}^{t}  \eta_k^2   \mE_{v_k} [\|\Pi_{k+1}^t(L_{K,\lambda}) \phi_{v_k} \|_{\rho}^2] \leq 2\kappa^2 (c_{\gamma} + \kappa^2) \eta(T)^{\gamma+1} T^{\gamma} \ln(\mathrm{2}T).
\eea
Condition \eref{etaT} ensures the right-hand side of the above is less than $1/2.$ This verifies \eref{restriction}. Thus, we can apply Proposition \ref{pro:totalerror_inter} to get \eref{totalerror_interm}.
Note that by Proposition \ref{pro:initialerror}, the initial error can be bounded as
\bea
\|S_1(t)\|_{\rho}^2 \leq R^2 \zeta^{2\zeta} \left(\eta(T)t\right)^{-2\zeta}.
\eea
 Plugging the above inequality and \eref{eq1} into \eref{totalerror_interm}, we derive
 \be\label{fixedStepExplict}
 \mE[\|r_{t+1}\|_2^2] \leq R^2 \zeta^{2\zeta} \left(\eta(T)t\right)^{-2\zeta} + 4(4\|f_{\rho}\|_{\rho}^2 + \mcE(f_{\rho}))\kappa^2  (c_{\gamma} + \kappa^2) \eta(T)^{\gamma+1} t^{\gamma} \ln(\mathrm{2}t),
 \ee
which leads to the desired result. The proof is complete.
\end{proof}

\begin{proof}[Proof of Corollary \ref{cor:eta}]
  We only need to verify \eref{etaT}.
  Since $\kappa^2 \geq 1$ and $\zeta >0,$ we know that ${\zeta \over 4\kappa^2 (c_{\gamma} + \kappa^2)(\zeta+1)} \leq 1,$ and consequently,
  $$\left({\zeta \over 4\kappa^2 (c_{\gamma} + \kappa^2)(\zeta+1)}\right)^{\gamma+1} \leq  {\zeta \over 4\kappa^2 (c_{\gamma} + \kappa^2)(\zeta+1)}.$$
  Therefore,
  \be\label{eq2}
  \eta(T)^{\gamma+1} T^{\gamma} \leq {\zeta \over 4\kappa^2 (c_{\gamma} + \kappa^2)(\zeta+1)} T^{-{(\gamma + 2\zeta)(\gamma+1) \over \gamma + 2\zeta + 1}} T^{\gamma} = {\zeta \over 4\kappa^2 (c_{\gamma} + \kappa^2)(\zeta+1)} T^{-{ 2\zeta \over \gamma + 2\zeta + 1}}.
  \ee
  Rewriting $T^{-{ 2\zeta \over \gamma + 2\zeta + 1}}$ as
  \bea
  2^{2\zeta \over \gamma + 2\zeta + 1} \exp\left\{-{ 2\zeta \over \gamma + 2\zeta + 1} \ln(2T)\right\},
  \eea
  and then applying \eref{exppoly} (from the appendix) with $c= { 2\zeta \over \gamma + 2\zeta + 1} \ln(2T)$, $x=1$ and $\zeta'=1,$ we know that
  \bea
  \eta(T)^{\gamma+1} T^{\gamma} \leq {\zeta \over 4\kappa^2 (c_{\gamma} + \kappa^2)(\zeta+1)}  2^{2\zeta \over \gamma + 2\zeta + 1} {\gamma+ 2\zeta+1 \over 2\zeta e \ln(2T)} \leq {1 \over 4\kappa^2 (c_{\gamma} + \kappa^2) \ln(2T)},
  \eea
  which leads to \eref{etaT}. Thus, following from the proof of Theorem \ref{thm:general_fixeta_un}, we get \eref{fixedStepExplict}. Plugging with
  \eref{etaTS} and \eref{eq2}, we get
  \bea
 \mE[\|r_{T+1}\|_2^2] \leq \left(R^2 ( 4\kappa^2(c_{\gamma}  +\kappa^2) (\zeta+1))^{2\zeta} + (4\|f_{\rho}\|_{\rho}^2 + \mcE(f_{\rho})){\zeta \over \zeta+1}\right) T^{-{2\zeta \over \gamma+ 2\zeta + 1}} \ln(\mathrm{2}T).
 \eea
  The proof is complete.
\end{proof}

\subsection{Case 2: $\lambda=0,\eta_1 = const, \theta>0$}

\begin{proof}[Proof of Theorem \ref{thm:general_theta}]
According to Proposition \ref{pro:TraceSumEst}, we have \eref{traceSumEst}, with $\mcF_{\eta_1,\theta, \lambda,\gamma}$ given by \eref{mcF}. When $\lambda=0$,
  \bea
 \mcF_{\theta,\eta_1,\lambda,\gamma} = \eta_1^{\gamma+1} t^{\gamma -\theta(\gamma + 1)} \ln(2t) \left( t^{(2\theta-1)_+} +1 \right) \leq 2\eta_1^{\gamma+1} t^{\gamma -\theta(\gamma + 1)+ (2\theta-1)_+} \ln(2t).
  \eea
Also, $2^{2\theta + \gamma - 1} \leq 2^{2\theta}$ since $\gamma \leq 1.$ Therefore,
  \bea
\sum_{k=1}^{t}  \eta_k^2   \mE_{v_k} [\|\Pi_{k+1}^t(L_{K,\lambda}) \phi_{v_k} \|_{\rho}^2] \leq (2^{2\theta}c_{\gamma} + \kappa^2) 2\eta_1^{\gamma+1} t^{\gamma -\theta(\gamma + 1) + (2\theta-1)_+}  \ln(2t).
\eea
Taking the maximum over $t \in [T]$ on both sides, multiplying both sides by $\kappa^2$, and recalling that $c_{\theta,\gamma}$ is given by \eref{cthetagamma},
\be\label{eq3}
\kappa^2\max_{t\in [T]}\sum_{k=1}^{t}  \eta_k^2   \mE_{v_k} [\|\Pi_{k+1}^t(L_{K,\lambda}) \phi_{v_k} \|_{\rho}^2] \leq 2\kappa^2(2^{2\theta}c_{\gamma} + \kappa^2) \eta_1^{\gamma+1} c_{\theta,\gamma}.
\ee
Condition \eref{eta1restriction} and $\kappa^2\geq 1$ imply that $\eta_1\leq 1$ and thus $\eta_1^{\gamma+1} \leq \eta_1$ since $\gamma \geq0.$ Therefore, by \eref{eta1restriction},
\bea
2\kappa^2(2^{2\theta}c_{\gamma} + \kappa^2) \eta_1^{\gamma+1} c_{\theta,\gamma} \leq 2\kappa^2(2^{2\theta}c_{\gamma} + \kappa^2) \eta_1 c_{\theta,\gamma} \leq 1/2.
\eea
Thus, \eref{restriction} holds. Now, we can apply Proposition \ref{pro:totalerror_inter} to obtain \eref{totalerror_interm}.
By Proposition \ref{pro:initialerror}, with $\lambda=0$, the initial error can be estimated as
\bea
\|S_1(t)\|_{\rho}^2 \leq (\zeta/\eta_1)^{2\zeta} R^2 t^{2\zeta(\theta-1)}.
\eea
Introducing the above and \eref{eq3} into \eref{totalerror_interm}, we get
\be\label{eq10}
\mE[\|r_{t+1}\|_{\rho}^2]
\leq (\zeta/\eta_1)^{2\zeta} R^2 t^{2\zeta(\theta-1)}  + 4(4\mcE(f_{\rho})
+ \|f_{\rho}\|_{\rho}^2)\kappa^2(2^{2\theta}c_{\gamma} + \kappa^2) \eta_1^{\gamma+1} t^{\gamma -\theta(\gamma + 1) + (2\theta-1)_+}  \ln(2t).
\ee
The proof is complete.
\end{proof}

\begin{proof}[Proof of Corollary \ref{cor:theta}]
	 We first need to prove \eref{eta1restriction}.
 If ${\gamma\over \gamma + 1}<\theta \leq 1/2$ and $\gamma \in[0,1)$, we have $2^{\theta} \leq \sqrt{2}$, and by \eref{exppoly} (from the appendix),
  $$
  (2t)^{\gamma - \theta(\gamma+1) } = \exp\left\{ -(\theta(\gamma+1) -\gamma) \ln (2t) \right\} \leq {1 \over e(\theta(\gamma+1) -\gamma) \ln(2t)},
  $$
  which implies
  $$c_{\theta,\gamma} = \max_{t\in [T]} \left\{ t^{\gamma - \theta(\gamma+1) } \ln (\mathrm{2}t) \right\} \leq {2^{\theta(\gamma+1) - \gamma} \over e(\theta(\gamma+1) -\gamma) } \leq {1 \over \sqrt{2}(\theta(\gamma+1) -\gamma)} .$$
 Therefore, \eref{eta1restriction} holds if
 \be\label{eq8}
 0< \eta_1 \leq { \theta(\gamma+1) -\gamma \over 3\kappa^2(2c_{\gamma} + \kappa^2)}.
 \ee

  When $2\zeta <1-\gamma,$ $\theta= {2\zeta + \gamma \over 2\zeta + \gamma + 1},$ and $\eta_1 = {\zeta \over 3\kappa^2(2c_{\gamma} + \kappa^2)},$ obviously,
 $$
 \theta = 1 - {1 \over 2\zeta + \gamma + 1} < {1 \over 2},
 $$
 and
 $$
 \eta_1 \leq {2\zeta \over 3\kappa^2(2c_{\gamma} + \kappa^2)(2\zeta + \gamma + 1)} = {\theta(\gamma+1) - \gamma \over 3\kappa^2(2c_{\gamma} + \kappa^2)}.
 $$
 This proves \eref{eq8} and consequently, \eref{eta1restriction}.
 Following the proof of Theorem \ref{thm:general_theta}, we thus have \eref{eq10}, which can be relaxed as
 \bea
\mE[\|r_{t+1}\|_{\rho}^2]
\leq \left\{(3\kappa^2(2c_{\gamma} + \kappa^2))^{2\zeta} R^2 + 2\zeta(4\mcE(f_{\rho})
+ \|f_{\rho}\|_{\rho}^2) \ln(2t) \right\} t^{-{2\zeta \over 2\zeta+ \gamma+1}} .
\eea
This leads to the first result of the theorem.

 When $2\zeta \geq 1 - \gamma$, $\theta = {1/2}$ and $\eta_1 = {1-\gamma \over 6\kappa^2 (2c_{\gamma} +\kappa^2)}.$ Condition \eref{eq8} is satisfied trivially.
  Thus, following the proof of Theorem \ref{thm:general_theta}, we have  \eref{eq10} and consequently,
  \bea
\mE[\|r_{t+1}\|_{\rho}^2]
\leq (6\kappa^2 (2c_{\gamma} +\kappa^2)\zeta/(1-\gamma))^{2\zeta} R^2 t^{-\zeta}  + (4\mcE(f_{\rho})
+ \|f_{\rho}\|_{\rho}^2) t^{\gamma - 1 \over 2}  \ln(2t),
\eea
which leads to the second result of the theorem by noting that
$t^{-\zeta} \leq  t^{\gamma - 1 \over 2}.$
The proof is complete.
\end{proof}

\subsection{Case 3: $\lambda>0,\eta_1 = const, \theta>0$}

\begin{proof}[Proof of Theorem \ref{thm:general_lambda}]
Following the proof of Theorem \ref{thm:general_theta}, we know that the condition \eref{restriction} from Proposition \ref{pro:totalerror_inter} is satisfied, and thus it holds that \eref{totalerror_interm}. Introducing \eref{initialerror} and \eref{traceSumEst} into \eref{totalerror_interm}, with $t=T,$ $\lambda = T^{\theta-1 +\epsilon},$ $\epsilon \in(0,1-\theta]$ and $\zeta \in(0,1],$ by a direct calculation, we get
\bea
\mE[\|r_{T+1}\|_{\rho}^2] \leq R^2\left( 1 + (\zeta/\eta_1)^{\zeta} \right)^2 T^{2 \zeta (\theta-1 +\epsilon)} + 2\kappa^2 (2^{2\theta}c_{\gamma} + \kappa^2) (4\|f_{\rho}\|_{\rho}^2 + \mcE(f_{\rho})) \mcF_{\eta_1,\theta, \lambda,\gamma}(T).
\eea
Recalling that $\mcF_{\eta_1,\theta, \lambda,\gamma}(t)$ is given by \eref{mcF}, by $\lambda = T^{\theta-1 +\epsilon}$ and \eref{exppoly} (from the appendix),
\bea
\exp\{-\eta_1\lambda T^{1-\theta} \} = \exp\{-\eta_1 T^{\epsilon} \} \leq \left( {(2\theta-1)_+ /\epsilon \over e\eta_1 T^{\epsilon}} \right)^{(2\theta-1)_+ /\epsilon} \leq
\left( {\epsilon e\eta_1 } \right)^{-(2\theta-1)_+ /\epsilon} T^{-(2\theta-1)_+} .
\eea
Thus,
\bea
\mcF_{\eta_1,\theta, \lambda,\gamma}(T) \leq \eta_1^{\gamma+1} T^{\gamma - \theta(\gamma+1)} \ln(2T) \left( \left( {\epsilon e\eta_1 } \right)^{-(2\theta-1)_+ /\epsilon} + 1\right).
\eea
It thus follows from the above analysis that
\begin{flalign*}
	\begin{split}
		&\mE[\|r_{T+1}\|_{\rho}^2] \leq R^2\left( 1 + (\zeta/\eta_1)^{\zeta} \right)^2 T^{2 \zeta (\theta-1 +\epsilon)}\\  &\quad \quad\quad+ 2\kappa^2(4\|f_{\rho}\|^2 + \mcE(f_{\rho}))(2^{2\theta}c_{\gamma} + \kappa^2) \eta_1^{\gamma+1} \left( \left( {\epsilon e\eta_1 } \right)^{-(2\theta-1)_+ /\epsilon} + 1\right) T^{\gamma - \theta(\gamma+1)} \ln(2T).
	\end{split}
\end{flalign*}
The proof is complete.
\end{proof}

\begin{proof}[Proof of Corollary \ref{cor:lambda}]
  We use Theorem \ref{thm:general_lambda}, with $\theta = {2\zeta + \gamma \over 2\zeta + \gamma + 1}$ and $\epsilon$ replaced by ${\epsilon \over 2\zeta},$ to prove the result. Obviously, we only need to prove the condition \eref{eta1restriction} is true. By a similar argument as that for \eref{eq8}, we know that \eref{eta1restriction} is true if
  \bea
  0< \eta_1 \leq {\theta(\gamma+1) - \gamma \over 3\kappa^2(2^{2\theta}c_{\gamma} +\kappa^2)}.
  \eea
  This can be verified by noting that
  $\theta(\gamma+1) - \gamma = {2\zeta \over 2\zeta + \gamma + 1} \geq {\zeta \over \zeta + 1}$ and $2^{2\theta}\leq 2^{3/2} \leq 3.$
  The proof is complete.
\end{proof}

\section*{Acknowledgement}
This material is based upon work supported by the Center for Brains, Minds and Machines (CBMM), funded by NSF STC award CCF-1231216. L. R. acknowledges the financial support of the Italian Ministry of Education, University and Research FIRB project RBFR12M3AC.  

\section*{Appendix}
	\begin{proof}[Proof of Lemma \ref{lemma:initialerror1}]
		Similar to the proof for Lemma \ref{lemma:initialerror2}, we have
		\bea
		\| \Pi_{k+1}^t(L_{K,\lambda}) L_K^{\zeta}\| = \sup_{i} \prod_{j=k+1}^t (1 - \eta_j(\lambda + \sigma_i))\sigma_i^{\zeta}.
		\eea
		Using the basic inequality
		\be\label{expx}
		1 + x \leq e^{x} \qquad \mbox{for all } x \geq -1, 
		\ee
		with $\eta_j(\lambda+ \sigma_i) \leq 1$, we get
		\begin{align*}
		\| \Pi_{k+1}^t(L_{K,\lambda}) L_K^{\zeta}\| \leq& \sup_i \exp\left\{ - (\lambda + \sigma_i) \sum_{j=k+1}^t \eta_j\right\} \sigma_i^{\zeta} \\
		\leq & \exp\left\{ - \lambda  \sum_{j=k+1}^t \eta_j\right\} \sup_{x \geq 0} \exp\left\{ - x  \sum_{j=k+1}^t \eta_j\right\} x^{\zeta}.
		\end{align*}
		The maximum of the function $g(x) = e^{-cx}x^{\zeta}$( with $c>0$) over $ \mR_+ $ is achieved at $x_{\max}= \zeta/c,$ and thus
		\be\label{exppoly}
		\sup_{x \geq 0} e^{-cx} x^{\zeta} =  \left({\zeta \over ec} \right)^{\zeta}.
		\ee
		Using this inequality, one can get the desired result \eref{initialerror1_interm}.
	\end{proof}

\bibliography{doubly_lin_rosasco}
\bibliographystyle{plain}

\end{document}